\documentclass[AMA,Times1COL]{WileyNJDv5}

\articletype{RESEARCH ARTICLE}

\startpage{1}

\usepackage{microtype}
\usepackage{subcaption}
\usepackage{tikz}
\usetikzlibrary{shapes,arrows,positioning,calc}

\begin{document}

\title{Koopman-Based Robust Model Predictive Control for Nonlinear Systems with Stochastic Intermittent Measurements}

\author[1,2]{Guanhua Liu$^\dagger$}

\author[1,2]{Tong Wu$^\dagger$}

\author[2]{Lixian Zhang}

\author[2]{Weifeng Du}

\author[3]{Minghao Han}

\authormark{LIU \textsc{et al.}}
\titlemark{KOOPMAN-BASED ROBUST MODEL PREDICTIVE CONTROL}

\address[1]{\orgdiv{State Key Laboratory of Robotics and Systems}, \orgname{Harbin Institute of Technology}, \orgaddress{\country{China}}}

\address[2]{\orgdiv{School of Astronautics}, \orgname{Harbin Institute of Technology}, \orgaddress{\country{China}}}

\address[3]{\orgdiv{Nanyang Environment and Water Research Institute (NEWRI)}, \orgname{Nanyang Technological University}, \orgaddress{\country{Singapore}}}

\corres{Corresponding author: Lixian Zhang, School of Astronautics, Harbin Institute of Technology, Harbin, China. E-mail: lixianzhang@hit.edu.cn \\[6pt]
\textsuperscript{$\dagger$}~Guanhua Liu and Tong Wu are co-first authors}

\abstract[Abstract]{Intermittent state measurements pose fundamental challenges to model predictive control of constrained nonlinear systems because
prediction uncertainty grows during feedback outages and measurement-triggered resets disrupt nominal state propagation, potentially
compromising closed-loop stability and recursive feasibility. This paper develops a Koopman-based stochastic MPC framework with
probabilistically truncated soft constraints. Specifically, a Lipschitz-constrained deep Koopman model provides a linear latent predictor,
enabling computationally efficient online optimization. The intermittent measurement process is modeled as a two-mode discrete-time
Markov chain, yielding a unified Markov jump error model for open-loop propagation and measurement-triggered resets. Under numerically
verifiable sufficient conditions, the prediction error is shown to be mean-square ultimately bounded, and an explicit uniform second-moment
bound is obtained. A distribution-free probabilistic error radius is then constructed for a prescribed confidence level and used to truncate
dropout-dependent constraint tightening. An exact-penalty soft-constraint mechanism accommodates reset-induced jumps and prolonged
dropouts. Under the stated terminal compatibility and bounded-disturbance conditions, recursive feasibility and mean-square ultimate
boundedness of the closed-loop regulation error are established. Numerical simulations on a visual-servoing tracking task corroborate these
theoretical results and demonstrate effective tracking under stochastic measurement unavailability.}



\maketitle

\section{Introduction}
\label{sec:introduction}
To systematically handle multi-variable state and control constraints, Model Predictive Control (MPC) has emerged as a premier framework owing to its proactive optimization-based nature \cite{nubert2020safe,xin2026stochastic}.
A standard assumption in conventional MPC paradigms is the continuous and uninterrupted availability of feedback, which may be violated in practical systems due to intermittent measurements caused by communication dropouts or detection failures~\cite{li2013output,liang2023tracking}.
A quintessential example of such intermittent feedback under stringent state boundaries is found in robotic navigation or visual servoing \cite{zhang2021robust}. In such systems, the feedback loop is closed via extracted image features.
However, this process frequently suffers from unpredictable target occlusions and narrow geometric limits dictated by the camera's field of view (FOV) \cite{hajiloo2015robust,yang2024rmpcbased}.
Guaranteeing that system trajectories respect safety boundaries in the presence of prolonged measurement dropouts necessitates the deployment of robust MPC mechanisms \cite{polver2025robust}.

Nonlinear Model Predictive Control (NMPC) has emerged as a prevailing solution for the control of constrained nonlinear systems.
Despite the theoretical appeal of robust NMPC, its application is severely hindered by the underlying non-convex optimization, which imposes prohibitive computational burdens on real-time solvers~\cite{kamath2026physicsinformed,wang2024physicsinformed}.
To overcome this computational complexity bottleneck, a linear modeling paradigm based on the Koopman operator has garnered widespread interest \cite{surana2016koopman,ke2025koopmanbased,duran-siguenza2025control,li2026recedinghorizon}.
By lifting the original nonlinear dynamics into a high-dimensional functional space, the Koopman operator provides a globally linear representation of the system, enabling the deployment of linear control techniques, thereby alleviating the computational burden \cite{rajkumar2025realtime,bruder2021advantages}.
However, traditional Koopman operator methods rely on the manual selection or heuristic design of basis functions, posing challenges for the construction of accurate approximations with theoretical guarantees for complex nonlinear systems.
Recent advancements in data-driven methodologies, particularly neural networks and deep learning, have facilitated the discovery of expressive latent dictionaries from data, enhancing representation accuracy and establishing the Deep Koopman control framework \cite{zhang2025koopmanbased,han2022desko,shi2022deep,ng2023datadriven,lyu2025koopmanbased,han2026mako,lian2026deep}.

To manage safety boundaries within linear systems, tube-based MPC has been extensively investigated \cite{hu2022robust,abbas2023linear,dong2025robustswitched,ping2022tubebased,zhang2022robust}.
Nevertheless, the aforementioned methods commonly rely on continuous feedback availability or worst-case error propagation.
Under prolonged stochastic measurement dropouts, the resulting uncertainty tubes may expand substantially, leading to conservative constraint tightening and possible loss of recursive feasibility.

To address the stochastic measurement unavailability, it is essential to characterize the evolutionary dynamics and impacts of feedback dropouts \cite{liang2023tracking,wu2024efficient}.
Stochastic MPC (SMPC) represents a prominent paradigm for enforcing safety control under random events, which leverages a probabilistically bounded framework, thereby alleviating the severe conservatism induced by the worst-case assumptions inherent to robust control paradigms \cite{bernardini2012stabilizing,mcallister2023inherent}.
Concurrently, incorporating data-driven modeling methodologies into the SMPC framework has attracted significant attention \cite{breschi2023datadriven}, further catalyzing the development of deep stochastic Koopman and switched Koopman control architectures \cite{han2022desko,ju2026model,qi2025learning}.
Notably, recent advances have integrated learning-based Koopman operators with Markov jump systems to capture stochastic mode transitions within a latent space~\cite{11312626}.
However, while Koopman operators have been integrated with robust and stochastic frameworks in recent advances, the simultaneous presence of complex nonlinearities,
stochastic measurement dropouts, and tight safety constraints has not been fully addressed.
How to preserve recursive feasibility of the Koopman-based MPC solver under prolonged measurement dropouts remains insufficiently explored.

Motivated by the above observations, this paper develops a Koopman-based stochastic MPC framework for constrained nonlinear systems
with intermittent state measurements. The proposed framework characterizes the prediction uncertainty induced by measurement loss and
recovery and incorporates the resulting probabilistic error bound into a soft-constrained MPC formulation with truncated constraint tightening.
The main contributions are as follows:

\begin{itemize}
    \item A unified Markov jump model is formulated for the Koopman prediction error, capturing measurement-triggered resets and open-loop
propagation in the presence of bounded model residuals and measurement noise.
    \item Numerically verifiable conditions are derived for the mean-square ultimate boundedness of the prediction error, together with an explicit
bound on its asymptotic second moment. Combined with the bounded initial error, this result yields a uniform-in-time second-moment bound
and, consequently, a distribution-free probabilistic error radius for a prescribed confidence level.
    \item A probabilistically truncated soft-constrained Koopman MPC scheme is developed to limit dropout-dependent constraint tightening and
accommodate measurement-triggered state resets. Under the stated terminal compatibility and bounded-disturbance conditions, recursive
feasibility and mean-square ultimate boundedness of the closed-loop regulation error are established, while the online optimization problem
remains convex.
\end{itemize}

\emph{Notation:}
$\mathbb{R}^n$, $\mathbb{R}^{n \times m}$, and $\mathbb{Z}_{\ge 0}$ denote the $n$-dimensional Euclidean space, the set of $n \times m$ real matrices, and the set of non-negative integers, respectively.
System variables are indexed by discrete-time $k \in \mathbb{Z}_{\ge 0}$ and future prediction steps $i \in \mathbb{Z}_{\ge 0}$.
$I$ denotes the identity matrix of appropriate dimensions.
For a vector $x$, $\|x\|_2$ and $\|x\|_\infty$ denote its $L_2$-norm and $L_\infty$-norm, respectively.
For a symmetric positive definite matrix $Q$, $\|x\|_Q^2 \triangleq x^T Q x$.
The notation $X \succ 0$ ($X \succeq 0$) implies that a symmetric matrix $X$ is positive definite (positive semi-definite).
$\lambda_{\min}(X)$ and $\lambda_{\max}(X)$ denote the minimum and maximum eigenvalues of a symmetric matrix $X$, respectively.
The operators $\oplus$ and $\ominus$ correspond to the Minkowski sum and Pontryagin set difference, respectively.
The mathematical expectation of a random variable is denoted by $\mathbb{E}[\cdot]$, and $\operatorname{Pr}(\cdot)$ represents the probability measure.
For a multivariate Gaussian distribution, $\mathcal{N}(\mu, \Sigma)$ denotes a normal distribution with mean $\mu$ and covariance $\Sigma$.

\section{Problem Formulation}\label{sec:problem_formulation}

\subsection{Koopman-based System Modeling and Nominal MPC Formulation}
\label{subsec:koopman_modeling}

Consider the discrete-time nonlinear dynamics of a system restricted by compact state boundaries:
\begin{equation}
    x_{k+1} = F(x_k, u_k)
\end{equation}
where $x_k \in \mathcal{X} \subset \mathbb{R}^{n_x}$ represents the true state vector in the state space, and $u_k \in \mathcal{U} \subset \mathbb{R}^{n_u}$ is the control input.

According to Koopman operator theory,
the dynamic evolution of a nonlinear autonomous system can be globally captured by the linear evolution of observable functions defined on the state space.
Let $\mathcal{H}$ denote an infinite-dimensional Hilbert space comprising scalar-valued observable functions $g: \mathcal{X} \rightarrow \mathbb{R}$.
To accommodate the control inputs, the Koopman operator $\mathcal{K}$ can be extended to an input-affine lifting framework.
Instead of embedding the control variables into the observable function, the control $u_k$ is treated as an exogenous forcing term.
Consequently, the evolution of any state-dependent observable $g(x_k) \in \mathcal{H}$ is mapped into a high-dimensional functional space via a state-only encoder, formulated as:
\begin{equation}\label{eq:Koopman}
    g(x_{k+1}) = (\mathcal{K} g)(x_k) \approx \mathcal{K}_A g(x_k) + \mathcal{K}_B u_k.
\end{equation}

Although the operator $\mathcal{K}$ preserves exact linearity,
its infinite-dimensional nature renders it intractable for system synthesis.
In practical applications,
it is necessary to identify a finite set of basis functions
$\mathbf{z}_k = [g_1(x_k), \dots, g_{n_z}(x_k)]^T \in \mathbb{R}^{n_z}$
to construct an approximate finite-dimensional linear invariant subspace.
However, for general nonlinear dynamics, deriving a set of basis functions that achieves lossless finite-dimensional truncation is difficult,
and the resulting residual approximation errors typically lack theoretical bounds.

To overcome this challenge,
a deep Koopman architecture \cite{shi2022deep} with Lipschitz constraints is introduced.
This architecture globally lifts the nonlinear system dynamics into a high-dimensional latent space
$\mathcal{Z} \subset \mathbb{R}^{n_z}$.
The latent state vector generated by the network encoder is defined as:
\begin{equation}\label{eq:encoder}
    z_k = \text{Encoder}(x_k) = [\phi(x_k)^T, \psi(x_k)^T]^T \in \mathbb{R}^{n_z}
\end{equation}

By introducing an identity skip-connection,
the first $n_x$ dimensions of the latent state are enforced to be identical to the original state,
formulated as $\phi(x_k) \equiv x_k$.
Consequently, the global Lipschitz constant of this coordinate mapping is locked at $L_\phi = 1$.
The remaining $n_{\psi}$ dimensions represent the nonlinear observable features,
$\psi(\cdot) \in \mathbb{R}^{n_\psi}$, which are parameterized by a lightweight multi-layer perceptron (MLP).

Spectral Normalization is applied across all hidden layers to maintain its global Lipschitz continuity.
By constraining the maximum singular value of each layer's weight matrix $W$ such that $\sigma_{\max}(W) = 1$,
and leveraging the 1-Lipschitz property of the ReLU activation function,
the global Lipschitz constant of the nonlinear feature mapping is bounded by $L_\psi \le 1$.

Thus, under the aforementioned structural consistency and spectral normalization constraints, the globally linearized nominal evolution equation in the latent space is expressed as:
\begin{equation}
    z_{k+1} = A z_k + B u_k
\end{equation}

To avoid the additional computational overhead of nonlinear decoding,
a constant sparse projection matrix $C = [I_{n_x}, \mathbf{0}]$ is employed,
where $\mathbf{0}$ is an $n_x \times n_\psi$ zero matrix.
Thereby, the reconstructed state can be obtained as $\hat{x}_k = C z_k = \phi(x_k) = x_k$,
which indicates that the original states are directly extracted from the high-dimensional latent space.

Optimization of the composite loss function $\mathcal{L} = \mathcal{L}_{pred} + \mathcal{L}_{eig} + \mathcal{L}_{ortho}$
is mandated to prevent overfitting and guarantee the dissipativity of the system matrix. Specifically:

\begin{enumerate}[1)]
    \item \textbf{Prediction Loss}:
    This term evaluates the dynamic evolution over a prediction horizon $N_p$:
    \begin{equation}
        \mathcal{L}_{pred} = \sum_{i=1}^{N_p} \gamma^i \left\| z_{k+i} - \left(A^i z_k + \sum_{j=0}^{i-1} A^{i-j-1} B u_{k+j}\right) \right\|_2^2,
    \end{equation}
    where $\gamma \in (0, 1]$ is a discount factor.

    \item \textbf{Eigenvalue Penalty}:
    To ensure that the system matrix $A$ preserves Schur stability, an eigenvalue penalty term is introduced to softly constrain its eigenvalues within the unit circle:
    \begin{equation}
         \mathcal{L}_{eig} = \alpha_{eig} \sum_{i} \max(0, |\lambda_i(A)| - \beta),
    \end{equation}
    where $\alpha_{eig} > 0$ represents the penalty weighting coefficient,
    and $\beta \in (0, 1)$ denotes the prescribed eigenvalue envelope boundary.

    \item \textbf{Orthogonal Regularization}:
    This term is introduced to suppress the transient growth induced by non-normality:
    \begin{equation}
        \mathcal{L}_{ortho} = \alpha_{ortho} \|AA^T - A^TA\|_F^2,
    \end{equation}
    where $\alpha_{ortho} > 0$ represents the regularization weighting coefficient, and $\|\cdot\|_F$ denotes the Frobenius norm.
\end{enumerate}

Furthermore, during the network initialization phase, the matrix $A$ is randomly generated and subjected to Singular Value Decomposition (SVD),
constructed as $A_{\text{init}} = U \Sigma_{\text{init}} V^T$.
The maximum singular value is truncated and scaled to be less than unity,
which guarantees that the initial spectral norm satisfies $\|A\|_2 < 1$,
thereby providing a Schur-stable starting point for the high-dimensional linear dynamics.
Subsequently, the eigenvalue penalty term $\mathcal{L}_{\text{eig}}$
encourages the matrix $A$ to preserve Schur stability during the nonlinear parameter optimization process.

Then, consider the MPC design for the original nonlinear system.
By leveraging the Koopman linear lifting framework \eqref{eq:Koopman},
the original NMPC problem is transformed into the latent space.
At control step $k$ with a prediction horizon $N_p$,
the nominal optimization problem $\mathcal{P}_N(\bar{z}_k)$ is formulated,
seeking the optimal control sequence $U^* = [u_{k|k}^{*T}, u_{k+1|k}^{*T}, \dots, u_{k+N_p-1|k}^{*T}]^T$ to minimize the quadratic cost:
\begin{subequations} \label{eq:mpc_constraints}
\begin{align}
    & \min_{U} \Bigg\{\sum_{i=0}^{N_p-1} \left( \| \bar{z}_{k+i|k} - z_{ref} \|_Q^2 + \| u_{k+i|k} \|_R^2 \right) + \| \bar{z}_{k+N_p|k} - z_{ref} \|_{P_f}^2 \Bigg\} \label{eq:mpc_cost} \\
    & \quad \text{s.t.} \nonumber \\
    & \qquad \bar{z}_{k+i+1|k} = A \bar{z}_{k+i|k} + B u_{k+i|k}, \quad i = 0, \dots, N_p-1 \label{eq:mpc_dyn} \\
    & \qquad \bar{z}_{k|k} = \bar{z}_k \\
    & \qquad x_{min} \le C \bar{z}_{k+i|k} \le x_{max}, \quad i = 1, \dots, N_p \\
    & \qquad \bar{z}_{k+N_p|k} \in \mathcal{Z}_f \\
    & \qquad u_{min} \le u_{k+i|k} \le u_{max}, \quad i = 0, \dots, N_p-1 \label{eq:Ucons}
\end{align}
\end{subequations}
where $Q \succ 0$ and $R \succ 0$ are symmetric positive definite weighting matrices;
$P_f \succ 0$ is the terminal penalty matrix derived from the discrete-time algebraic Riccati equation (DARE);
$\mathcal{Z}_f$ is the nominal terminal invariant set;
$\bar{z}_{k+i|k}$ is the nominal predicted state in the latent space;
and $z_{ref}$ is the reference state mapped into the latent space.

To simplify the algebraic derivation,
a  coordinate translation is assumed to map the target reference to the origin,
establishing $\phi(x_{ref}) = \mathbf{0}$.
Concurrently, the regularization training of the deep network ensures that its nonlinear feature mapping closely approximates
$\psi(\mathbf{0}) = \mathbf{0}$.

\subsection{Error Envelope Bounding in the Latent Space}
\label{subsec:gaussian_uncertainty}

Under stochastic measurement intermittency, the true state cannot be directly acquired.
The unmodeled dynamics, target maneuvers, or environmental fluctuations generate an unknown additive drift disturbance $w_k \in \mathbb{R}^{n_x}$.
Consequently, the evolution of the nonlinear system is governed by:
\begin{equation}
    x_{true, k+1} = F(x_{true, k}, u_k) + w_k
\end{equation}

In practical constrained nonlinear control systems,
since the instantaneous rate of state change is limited by the system dynamics,
the actual drift is confined within a compact operational envelope.
The disturbance $w_k$ is modeled as a truncated multivariate Gaussian distribution defined over a compact set,
whose probability density function $p(w_k)$ satisfies:
\begin{equation}
    p(w_k) = \begin{cases}
        \frac{1}{\eta} \mathcal{N}(w_k \mid \mathbf{0}, \Sigma_{w}), & \text{if } \|w_k\|_2 \le r_w \\
        0, & \text{otherwise}
    \end{cases}
\end{equation}
where $\|w_k\|_2 \le r_w$ represents the closed ball determined by the maximum single-step drift capability, and the constant
$\eta = \int_{\|\xi\|_2 \le r_w} \mathcal{N}(\xi \mid \mathbf{0}, \Sigma_{w}) d\xi$
is the partition function ensuring normalization over the integration domain.
Since this geometric truncation confines the drift within a bounded envelope,
it guarantees the global boundedness of the expected squared norm of the disturbance.

Accounting for the measurement unavailability,
the maximum prediction error bound in the worst case scenario at the $i$-th prediction step is denoted as $r_{x, i}$.
This boundary quantifies the cumulative error envelope in the state space,
arising from the combined effects of the unknown drift and the model mismatch.
Consequently, the low-dimensional total error of the true state $x_{true, k+i}$ deviating from the nominal predicted state $\bar{x}_{k+i|k} \triangleq C \bar{z}_{k+i|k}$ is bounded in the worst case by:
\begin{equation}\label{eq:rxi}
    \|x_{true, k+i} - \bar{x}_{k+i|k}\|_2 \le r_{x, i}
\end{equation}

Since the nominal predicted state $\bar{z}_{k+i|k}$ is recursively generated by the globally linearized model,
its connection to the nonlinear observation function and the corresponding approximation error bounds are governed by the following assumption:

\begin{assumption}\label{ass:structural_consistency}
For the nominal predicted state $\bar{z}$ within the latent space, the latent consistency residual $\Delta \varepsilon_{rec} \triangleq \text{Encoder}(\bar{x}) - \bar{z}$ is uniformly bounded.
Specifically, maintaining the identity mapping $\phi(\bar{x}) \equiv \bar{x}$,
the reconstruction decomposition satisfies
$\bar{z} = \text{Encoder}(\bar{x}) - \Delta \varepsilon_{rec} = [\bar{x}^T, \psi(\bar{x})^T]^T - [\mathbf{0}^T, \Delta \varepsilon_{\psi}^T]^T$,
where the residual obeys $\|\Delta \varepsilon_{rec}\|_2 \le \varepsilon_{rec}$ for a small positive constant $\varepsilon_{rec}$.
\end{assumption}

The uniform boundedness assumed in Assumption \ref{ass:structural_consistency} is practically well-posed.
As established in \cite{zhang2022robust},
for an observable mapping that is Lipschitz continuous,
its residual function evaluated over a compact domain possesses a finite upper bound.
Specifically, let the state operational space $\mathcal{X} \subset \mathbb{R}^{n_x}$ be a compact set.
Given that the nonlinear feature mapping $\psi(\cdot)$ is enforced to be globally Lipschitz continuous via spectral normalization,
the latent consistency residual mapping $\Delta \varepsilon_{rec}: \mathcal{X} \to \mathbb{R}^{n_z}$ is continuous.
Consequently, its image set $\Delta \mathcal{E} = \{\Delta \varepsilon_{rec}(x) \mid x \in \mathcal{X}\}$ is compact.
This guarantees the existence of a finite radius $\varepsilon_{rec} < \infty$ such that $\sup_{e \in \Delta \mathcal{E}} \Vert{}e\Vert{}_2 \le \varepsilon_{rec}$.

By leveraging Assumption~\ref{ass:structural_consistency} alongside \eqref{eq:rxi},
the low-dimensional cumulative tracking error is mapped into the high-dimensional latent space.
Based on this mechanism, the following lemma is proposed to determine the envelope for the latent prediction error.

\begin{lemma}
\label{lem:latent_error_bound}
Let the prediction error in the latent space be $e_{z, i} \triangleq z_{true, k+i} - \bar{z}_{k+i|k}$.
Then, $e_{z, i}$ is bounded by the compact error envelope set $\mathcal{E}_{z, i}$:
\begin{equation}
e_{z, i} \in \mathcal{E}_{z, i} \triangleq \Big\{ e_z \in \mathbb{R}^{n_z} \mid \|e_z\|_2 \le \sqrt{2} r_{x, i} + \varepsilon_{rec} \Big\}
\end{equation}
which implies the state inclusion $z_{true, k+i} \in \bar{z}_{k+i|k} \oplus \mathcal{E}_{z, i}$.
\end{lemma}

\begin{proof}
By definition of the nonlinear lifting mapping,
the true latent state is $z_{true, k+i} = [\phi(x_{true, k+i})^T, \psi(x_{true, k+i})^T]^T$.
Utilizing Assumption~\ref{ass:structural_consistency},
the nominal predicted state can be decomposed as
$\bar{z}_{k+i|k} = [\phi(\bar{x}_{k+i|k})^T, \psi(\bar{x}_{k+i|k})^T]^T - \Delta \varepsilon_{rec}$.
Consequently, the high-dimensional latent space error vector can be partitioned into the low-dimensional encoding error
$e_\phi = \phi(x_{true, k+i}) - \phi(\bar{x}_{k+i|k})$,
the nonlinear feature error $e_\psi = \psi(x_{true, k+i}) - \psi(\bar{x}_{k+i|k})$,
and the latent consistency residual $\Delta \varepsilon_{rec}$.

Let the true prediction error in the original state space be defined as $e_x = x_{true, k+i} - \bar{x}_{k+i|k}$.
Since the coordinate identity mapping and the spectral normalization constraint ensure $L_\phi = 1$ and $L_\psi \le 1$, respectively,
the Lipschitz continuity of the encoder implies that $\|e_\phi\|_2 = \|e_x\|_2$ and $\|e_\psi\|_2 \le \|e_x\|_2$.
By the orthogonal decomposition property of the high-dimensional vector $L_2$ norm and the triangle inequality, it follows that:
$\|e_{z, i}\|_2 \le \sqrt{\|e_\phi\|_2^2 + \|e_\psi\|_2^2} + \|\Delta \varepsilon_{rec}\|_2
\le \sqrt{1^2 + 1^2}\, \|e_x\|_2 + \varepsilon_{rec}
= \sqrt{2}\, \|e_x\|_2 + \varepsilon_{rec}.$

Substituting the known state space error bound $\|e_x\|_2 \le r_{x, i}$ yields $\|e_{z, i}\|_2 \le \sqrt{2} r_{x, i} + \varepsilon_{rec}$.
\end{proof}

Furthermore, to characterize the disturbances and nominal state resets in the subsequent Markov jump framework,
it is necessary to define the upper bounds of the single-step additive perturbations in the high-dimensional latent space.
For the open-loop dynamic drift during measurement unavailability,
let $r_w$ be the single-step drift upper bound, i.e., $\|w_k\|_2 \le r_w$.
The compact sets formed here represent safety boundaries containing all possible realizations, rather than probabilistic confidence intervals.

\begin{lemma}
\label{lem:single_step_bounds}
Under Assumption~\ref{ass:structural_consistency},
the single-step perturbations mapped into the high-dimensional latent space satisfy the following bounds:
\begin{enumerate}[1)]
    \item Given a single-step upper bound $\varepsilon_{model} \ge 0$ for the dynamic mismatch residual \cite{han2026mako},
    the additive dynamic disturbance $d_k \triangleq z_{true, k+1} - (A z_{true, k} + B u_k)$ is bounded by the compact disturbance set $\mathcal{D}$:
    \begin{equation}
        d_k \in \mathcal{D} \triangleq \Big\{d \in \mathbb{R}^{n_z} \mid \|d\|_2 \le \sqrt{2} r_w + \varepsilon_{model} \Big\}
    \end{equation}

    \item The observation is corrupted by a bounded sensor noise $v_{sensor, k}$
    satisfying $\|v_{sensor, k}\|_2 \le r_{sensor}$,
    such that $x_{meas, k} = x_{true, k} + v_{sensor, k}$.
    The state reset noise $v_k \triangleq z_{true, k} - \text{Encoder}(x_{meas, k})$ is bounded by the compact noise set $\mathcal{E}_0$:
    \begin{equation}
        v_k \in \mathcal{E}_0 \triangleq \Big\{v \in \mathbb{R}^{n_z} \mid \|v\|_2 \le \sqrt{2} r_{sensor} \Big\}
    \end{equation}
\end{enumerate}
\end{lemma}

\begin{proof}
The derivations of these boundaries rely on the triangle inequality and Lipschitz continuity,
analogous to Lemma~\ref{lem:latent_error_bound}, and are omitted here for brevity.
\end{proof}

The existence of the finite upper bound $\varepsilon_{model}$ introduced in Lemma \ref{lem:single_step_bounds} is well-posed.
Analogous to the residual bounding in~\cite{zhang2022robust},
the dynamic mismatch $d_k$ can be reformulated as the difference between the true nonlinear lifting transition and the nominal linear Koopman prediction.
Specifically, since both the state operational space $\mathcal{X} \subset \mathbb{R}^{n_x}$ and the control space $\mathcal{U} \subset \mathbb{R}^{n_u}$ constitute compact sets,
and the constrained Deep Koopman encoder preserves Lipschitz continuity,
this difference mapping is continuous over a compact domain $\mathcal{X} \times \mathcal{U}$.
Therefore, the residual norm is bounded by a finite constant $\varepsilon_{model} \triangleq \sup_{x \in \mathcal{X}, u \in \mathcal{U}} \| \text{Encoder}(F(x, u)) - A \text{Encoder}(x) - B u \|_2$.

Since the sets $\mathcal{D}$ and $\mathcal{E}_0$ are both closed balls with finite radii,
the random disturbance $d_k$ and reset noise $v_k$ are bounded in magnitude.
Consequently, their expected squared norms possess finite upper bounds;
that is, there exist finite positive scalar constants $M_d < \infty$ and $M_v < \infty$
such that the stochastic tracking error system satisfies $\mathbb{E}[\|d_k\|^2] \le M_d$ and $\mathbb{E}[\|v_k\|^2] \le M_v$.

\subsection{Stochastic Switching Modeling for Intermittent Measurements}
\label{subsec:markov_jump_modeling}

To analyze the measurement discontinuity,
the sequence of measurement modes is modeled as a discrete-time Markov chain with a mode variable $\theta_k \in \mathcal{S} = \{0, 1\}$.
Here, $\theta_k = 0$ indicates the measurement is available but corrupted by noise, while $\theta_k = 1$ indicates the measurement unavailability.
The state transition probability matrix between different measurement availability modes is defined as
$\mathbb{P} = [p_{ij}] \in \mathbb{R}^{2 \times 2}$, where $p_{ij} \triangleq \operatorname{Pr}(\theta_{k+1} = j \mid \theta_k = i)$,
satisfying the row normalization condition $p_{i0} + p_{i1} = 1$.

Under the \textbf{measurement-triggered reset} mechanism,
true error states cannot be acquired during measurement dropouts,
and the evolution of the nominal state relies on the availability of reset.
In this mechanism, the evolution trajectory of the latent prediction error $e_{k+1}$ is dictated by whether a valid measurement is obtained at step $k+1$.
The specific mode-dependent state update and error propagation logic is detailed as follows:
\begin{enumerate}[1)]
    \item \textbf{If $\theta_{k+1} = 0$ :}
    Valid state measurement is restored at step $k+1$.
    The system utilizes the nonlinear network encoder to reset the current nominal predicted state,
    which is executed as $\bar{z}_{k+1|k+1} = \text{Encoder}(x_{meas, k+1})$.
    Consequently, the accumulated prediction error during the measurement dropout period is instantaneously cleared,
    leaving only the reset noise such that $e_{k+1} = v_{k+1}$.
    The compact envelope accounts for both the sensor noise and the encoder's latent consistency residual.
    \item \textbf{If $\theta_{k+1} = 1$ :}
    In the absence of valid measurement feedback,
    the controller continues to utilize the open-loop nominal predicted state.
    The prediction error between the true and nominal states propagates in the latent space according to: $e_{k+1} = A e_k + d_k$.
\end{enumerate}

By introducing a system matrix modulated by the next step mode $\theta_{k+1}$,
this complex behavior encompassing both dynamic evolution and resets is unified into the following nonhomogeneous Markov jump error dynamics:
\begin{equation}\label{eq:mjls_error_dynamics}
    e_{k+1} = A(\theta_{k+1}) e_k + w(\theta_{k+1}, k)
\end{equation}
where the system matrix $A(\theta_{k+1})$ and the equivalent additive disturbance term $w(\theta_{k+1}, k)$ depending on the system availability mode are defined as:
\begin{subequations}
\begin{align}
    A(\theta_{k+1}) &= \begin{cases}
        \mathbf{0}, & \text{if } \theta_{k+1} = 0 \\
        A, & \text{if } \theta_{k+1} = 1
    \end{cases}, \\
    w(\theta_{k+1}, k) &= \begin{cases}
        v_{k+1}, & \text{if } \theta_{k+1} = 0 \\
        d_k, & \text{if } \theta_{k+1} = 1
    \end{cases}\label{eq:equivalent_disturbance}
\end{align}
\end{subequations}
As established in Lemma~\ref{lem:single_step_bounds},
both the state reset noise $v_{k+1} \in \mathcal{E}_0$ and the single-step dynamic disturbance $d_k \in \mathcal{D}$
reside within compact sets with finite radii in the high-dimensional latent space.
Consequently, the expected squared norm of this jump disturbance possesses a deterministic global finite upper bound under any system mode,
facilitating the subsequent analysis of mean-square ultimate boundedness for the overall error system.

\section{Main Results}
\label{sec:main_results}

\subsection{Open-Loop Error Envelope Expansion}
\label{subsec:open_loop_expansion}

Suppose the system experiences a measurement dropout at time $t_c$, triggering an open-loop evolution for $k \in [t_c, t_c + T_d - 1]$.
The nominal state evolves as $\bar{z}_{k+1} = A \bar{z}_k + B u_k$,
while the true state incorporates the additive dynamic disturbance $d_k \in \mathcal{D}$ as $z_{true, k+1} = A z_{true, k} + B u_k + d_k$.
Define the accumulated open-loop tracking error as $e_{track,k} \triangleq z_{true, k} - \bar{z}_k$.
It is assumed that at the onset of the measurement unavailability sequence,
the initial tracking error is confined to the compact set $\mathcal{E}_0$ due to historical observation noise,
satisfying $e_{track,t_c} \in \mathcal{E}_0$.
Subtracting the nominal dynamics from the true dynamics yields the linear recursive equation for the prediction error during continuous measurement dropout:
\begin{equation}
    e_{track,k+1} = A e_{track,k} + d_k
\end{equation}

After $l$ continuous dropout steps ($1 \le l \le T_d$),
the high-dimensional error envelope $\mathcal{E}_l \subset \mathbb{R}^{n_z}$ and its low-dimensional projection $\mathcal{R}_l \subset \mathbb{R}^{n_x}$
via the observation matrix $C$ are derived as:
\begin{equation}
    \mathcal{R}_l \triangleq C \mathcal{E}_l = C A^l \mathcal{E}_0 \oplus \bigoplus_{j=0}^{l-1} C A^j \mathcal{D}
\end{equation}

Let $e_{x} \in \mathcal{R}_l$ denote an arbitrary projection error vector.
To dynamically tighten the constraint boundaries,
the maximum projection radii of the error set $\mathcal{R}_l$ along each coordinate axis are computed as:
\begin{equation}
    \Delta \mathcal{X}_l^{(j)} \triangleq \max_{e_x \in \mathcal{R}_l} |c_j^T e_x|, \quad j = 1, \dots, n_x
\end{equation}
where $c_j^T$ denotes the $j$-th row of $I_{n_x}$.
Adopting the independent coordinate contraction radii $\Delta \mathcal{X}_l^{(j)}$
corresponds to bounding the projected error set $\mathcal{R}_l$ with its $L_\infty$-norm outer hypercube envelope.

\begin{remark}
\label{rem:conservative_approximation}
This $L_\infty$-norm relaxation constitutes a conservative inner approximation of the exact nominal safe constraint set $\mathcal{X}_{safe} \ominus \mathcal{R}_l$.
The primary objective is to convert computationally prohibitive non-polyhedral set difference operations into standard axis-aligned linear constraints,
preserving theoretical safety while satisfying the real-time efficiency requirements ~\cite{arcari2023stochastic,tan2023distributionally,villanueva2024configurationconstrained}.
\end{remark}

To guarantee safety within the operational boundaries
$\mathcal{X}_{safe} \triangleq \{ x \in \mathbb{R}^{n_x} \mid x_{\min}^{(j)} \le c_j^T x \le x_{\max}^{(j)}, \ j=1,\dots,n_x \},$
the nominal predicted state must satisfy the tightened constraints:
\begin{equation}
    x_{\min}^{(j)} + \Delta \mathcal{X}_l^{(j)} \le c_j^T C \bar{z}_{t_c+l|t_c} \le x_{\max}^{(j)} - \Delta \mathcal{X}_l^{(j)}, \quad j = 1, \dots, n_x
\end{equation}

As the measurement dropout duration increases,
this continuous constraint tightening restricts the admissible operating domain,
forcing the nominal state trajectory to violate the constraints, thereby compromising the recursive feasibility.

\subsection{Mean-Square Ultimate Boundedness of the Prediction Error}
\label{subsec:msub_markov_jumps}

To evaluate the global dissipativity of the Markov jump error dynamics,
a stochastic quadratic Lyapunov function,
parameterized by the current measurement availability mode $\theta_k \in \mathcal{S} = \{0, 1\}$,
is constructed as follows:
\begin{equation}
    V(e_k, \theta_k) = e_k^T P(\theta_k) e_k
\end{equation}
where the weighting matrices are required to be positive definite for all possible modes, satisfying $P(i) \succ 0$ for $i \in \{0, 1\}$.

According to the Rayleigh quotient property of quadratic forms,
this Lyapunov function is bounded in the state space by the squared $L_2$ norm of the error state.
Defining the global minimum and maximum eigenvalue constants as
$c_1 \triangleq \min_{i \in \{0, 1\}} \lambda_{\min}(P(i)) > 0$ and $c_2 \triangleq \max_{i \in \{0, 1\}} \lambda_{\max}(P(i)) > 0$,
the following envelope relation holds for any time step $k$ and mode $\theta_k$:
\begin{equation}\label{eq:eqLyabound}
    c_1 \|e_k\|_2^2 \le V(e_k, \theta_k) \le c_2 \|e_k\|_2^2
\end{equation}

Next, the single-step conditional expectation difference of the Lyapunov function along the true system trajectory is evaluated:
\begin{equation}
    \mathbb{E}[ \Delta V \mid e_k, \theta_k=i ] = \mathbb{E}[ V_{k+1} \mid e_k, \theta_{k} =i ] - V(e_k, i)
\end{equation}

By applying the law of total probability and substituting the reconstructed error dynamics ~\eqref{eq:mjls_error_dynamics},
the expectation term $\mathbb{E}[ V_{k+1} \mid e_k, \theta_{k} = i ]$
is expanded across the two possible modes at time $k+1$:
an instantaneous measurement-triggered reset when $\theta_{k+1}=0$,
and open-loop propagation when $\theta_{k+1}=1$.
\begin{equation}
    \mathbb{E}[ V_{k+1} \mid e_k, \theta_{k} = i ] = p_{i0} \mathbb{E}[ e_{k+1}^T P(0) e_{k+1} \mid \theta_{k+1} = 0 ] + p_{i1} \mathbb{E}[ e_{k+1}^T P(1) e_{k+1} \mid e_k, \theta_{k+1} = 1 ]
\end{equation}

In this expansion, when $\theta_{k+1}=0$,
the system undergoes a reset,
reducing the new error to the reset noise,
yielding $e_{k+1} = v_{k+1}$.
In this scenario, the prediction error is decoupled from the current state $e_k$.
Conversely, when $\theta_{k+1}=1$, the system propagates the open-loop nominal dynamics,
and the error evolves as $e_{k+1} = A e_k + d_k$. Substituting these relations yields:
\begin{equation}\label{eq:deltaLyapunov}
\mathbb{E}[ V_{k+1} \mid e_k, \theta_{k} = i ] = p_{i0} \mathbb{E}[ v_{k+1}^T P(0) v_{k+1} ] + p_{i1} \mathbb{E}[ (A e_k + d_k)^T P(1) (A e_k + d_k) \mid e_k ]
\end{equation}

Expanding the quadratic form in the second term generates the cross term $2 (A e_k)^T P(1) d_k$.
Given that the equivalent dynamic disturbance $d_k$,
arising from nonlinear unmodeled dynamics and continuous dynamic mismatches,
exhibits strong state dependency,
statistical independence between $d_k$ and $e_k$ cannot be assumed.
Therefore, by applying Young's inequality under a weighted inner product,
this cross term can be upper-bounded for any given positive real constant $\zeta > 0$ as follows:
\begin{equation}\label{eq:Youngs}
    2 (A e_k)^T P(1) d_k \le \zeta e_k^T A^T P(1) A e_k + \frac{1}{\zeta} d_k^T P(1) d_k
\end{equation}

Bounding the cross term in the expanded conditional expectation \eqref{eq:deltaLyapunov}
using the right side of \eqref{eq:Youngs} decouples the quadratic terms containing $e_k$ from the disturbance expectation terms.
Accordingly, the mode-dependent equivalent expected disturbance bound $M_w(i)$ is defined as:
\begin{equation}
    M_w(i) \triangleq p_{i0} \mathbb{E}[ v_{k+1}^T P(0) v_{k+1} ] + p_{i1} \left(1 + \frac{1}{\zeta}\right) \mathbb{E}[ d_k^T P(1) d_k ]
\end{equation}

As established in Lemma \ref{lem:single_step_bounds},
both the state reset noise $v_{k+1} \in \mathcal{E}_0$
and the single-step dynamic disturbance $d_k \in \mathcal{D}$
are confined to compact sets with finite radii.
Consequently, their expected squared norms possess upper bounds under any mode:
\begin{subequations}
\begin{align}
    \mathbb{E}[ v_{k+1}^T P(0) v_{k+1} ] &\le \lambda_{\max}(P(0)) \sup_{v \in \mathcal{E}_0} \|v\|_2^2 < \infty \\
    \mathbb{E}[ d_k^T P(1) d_k ] &\le \lambda_{\max}(P(1)) \sup_{d \in \mathcal{D}} \|d\|_2^2 < \infty
\end{align}
\end{subequations}

Hence, there definitively exists a global finite scalar constant $M_w > 0$,
independent of the current state $e_k$ and observation mode $i$,
such that the envelope $M_w(i) \le M_w$ holds for all $i \in \{0,1\}$.
In summary, the single-step conditional expectation difference inequality of the Lyapunov function can be bounded and simplified to:
\begin{equation}\label{eq:ExpecteddifferenceInequality}
\mathbb{E}[ \Delta V \mid e_k, \theta_k=i ] \le e_k^T \left( p_{i1}(1+\zeta) A^T P(1) A \right) e_k - e_k^T P(i) e_k + M_w
\end{equation}

To ensure that the jump system maintains energy dissipativity under continuous random disturbances, the quadratic matrix within the parentheses must be negative definite.
This yields the coupled Linear Matrix Inequalities (LMIs) that govern the mean-square stability of the measurement-triggered system:
\begin{equation}
    p_{i1}(1+\zeta) A^T P(1) A - P(i) \prec 0, \quad \forall i \in \{0, 1\}
\end{equation}

\begin{definition}
\label{def:msub}
For the stochastic Markov jump error dynamics established in \eqref{eq:mjls_error_dynamics},
the high-dimensional prediction error sequence is said to achieve global \textbf{Mean-Square Ultimate Boundedness (MSUB)} if,
for any initial error state $e_0 \in \mathbb{R}^{n_z}$ and any initial measurement availability mode $\theta_0 \in \mathcal{S}$,
there exists a global finite positive constant $E_{\infty} > 0$ such that the expected squared norm of the error state satisfies:
\begin{equation}
    \limsup_{k \to \infty} \mathbb{E} [ \|e_k\|_2^2 ] \le E_{\infty}^2
\end{equation}
\end{definition}

\begin{theorem}
\label{thm:msub}
Suppose the nominal latent space system matrix $A$ is Schur stable,
and the transition probability satisfies $p_{11} < 1$.
Then, the high-dimensional prediction error sequence $e_k$ is global MSUB, i.e., the following condition holds:
\begin{equation}
    \limsup_{k \to \infty} \mathbb{E} [ \|e_k\|_2^2 ] \le \frac{c_2 M_w}{c_1 \alpha} \triangleq E_{\infty}^2
\end{equation}
where $\alpha > 0$ is the global dissipation rate,
and $E_{\infty}^2$ is the steady-state mean-square limit.
\end{theorem}

\begin{proof}

The nominal matrix $A$ has been endowed with Schur stability during the network training phase.
To guarantee the desired expected dissipativity under this framework,
it is necessary to prove that the coupled inequality group $p_{i1}(1+\zeta) A^T P(1) A - P(i) \prec 0$
admits positive definite matrix solutions $P(i) \succ 0$ for both $i \in \{0,1\}$.

First, consider the case when the system is in the measurement dropout mode ($i=1$). The stability condition reduces to:
\begin{equation}
    p_{11}(1+\zeta) A^T P(1) A - P(1) \prec 0
\end{equation}

The existence of a positive definite solution $P(1) \succ 0$ for this LMI is guaranteed if and only if the scaled matrix $\sqrt{p_{11}(1+\zeta)}A$ is Schur stable.
This stability requirement is equivalent to ensuring that its spectral radius is strictly within the unit circle, yielding:
\begin{equation}\label{eq:relaxed_zeta_condition}
    p_{11}(1+\zeta)\rho(A)^2 < 1
\end{equation}

Given the transition probability condition $p_{11} < 1$ and that the dissipativity-constrained Deep Koopman network ensures a stable nominal spectral radius $\rho(A) < 1$,
the upper bound $1/\rho(A)^2$ is greater than unity.
Consequently, it is always feasible to select a positive scalar $\zeta > 0$ such that the contraction criterion \eqref{eq:relaxed_zeta_condition} is fully satisfied.
Under this valid parameter assignment, the discrete-time algebraic Lyapunov equation $P(1) - p_{11}(1+\zeta) A^T P(1) A = Q$ admits a unique positive definite matrix solution $P(1) \succ 0$ for any given symmetric matrix $Q \succ 0$, thereby proving the feasibility of the LMI for $i=1$.

Next, consider the case when the system is in the measurement available mode ($i=0$). The stability condition becomes:
\begin{equation}
    p_{01}(1+\zeta) A^T P(1) A - P(0) \prec 0
\end{equation}

From the preceding step,
the positive definite matrix $P(1) \succ 0$ is fully determined.
Consequently, the term $S \triangleq p_{01}(1+\zeta) A^T P(1) A$ constitutes a positive semi-definite bounded matrix.
Thus, setting $P(0) = S + \mu I$,
where $\mu > 0$ is an arbitrary positive scalar,
unconditionally guarantees $P(0) \succ 0$ and satisfies the inequality.

Given the feasibility of the coupled LMIs,
there exists a global positive scalar decay rate $\alpha$, defined by the minimum eigenvalue across all modes:
\begin{equation}\label{eq:EigBound}
    \alpha \triangleq \min_{i \in \{0, 1\}} \lambda_{\min}\big(P(i) - p_{i1}(1+\zeta) A^T P(1) A\big) > 0
\end{equation}

Applying this eigenvalue bound \eqref{eq:EigBound} to the expected difference inequality derived earlier,
the single-step conditional expectation of the Lyapunov function satisfies the following dissipation condition:
\begin{equation}
    \mathbb{E} [ \Delta V \mid e_k, \theta_k=i ] \le -\alpha \|e_k\|_2^2 + M_w
\end{equation}
Taking the unconditional expectation on both sides over all possible modes and states yields:
\begin{equation}
    \mathbb{E}[V_{k+1}] - \mathbb{E}[V_k] \le -\alpha \mathbb{E}[\|e_k\|_2^2] + M_w
\end{equation}

Incorporating the equivalent Lyapunov bounds \eqref{eq:eqLyabound}, expressed as $-\|e_k\|_2^2 \le -\frac{1}{c_2} V_k$, yields the geometric drift condition:
\begin{equation}
    \mathbb{E}[V_{k+1}] \le \left(1 - \frac{\alpha}{c_2}\right) \mathbb{E}[V_k] + M_w
\end{equation}

Since the global dissipation rate $\alpha > 0$ and the maximum eigenvalue constant $c_2 > 0$,
the geometric drift contraction condition $1 - \frac{\alpha}{c_2} < 1$ is satisfied.
As the time step $k \to \infty$, the global expectation of the Lyapunov function is ultimately bounded by a constant limit:
\begin{equation}
    \limsup_{k \to \infty} \mathbb{E} [ V_k ] \le \frac{M_w}{\alpha / c_2} = \frac{c_2 M_w}{\alpha}
\end{equation}

Re-applying the lower-bound inequality $\mathbb{E}[\|e_k\|_2^2] \le \frac{1}{c_1} \mathbb{E}[V_k]$
establishes the ultimate steady-state upper bound for the expected high-dimensional error norm:
\begin{equation}
    \limsup_{k \to \infty} \mathbb{E} [ \|e_k\|_2^2 ] \le \frac{c_2 M_w}{c_1 \alpha} \triangleq E_{\infty}^2
\end{equation}
\end{proof}

Theorem \ref{thm:msub} establishes the global convergence of the system under non-homogeneous Markov jump disturbances.
This implies that although the open-loop error envelope experiences severe local inflation during a single prolonged sequence of measurement dropout,
from a macroscopic statistical time domain perspective,
provided the objective system permits intermittent measurement updates satisfying $1 - p_{11} > 0$,
the mean-square value of the prediction error is prevented from unbounded divergence.

\begin{remark}
\label{rem:limit_behavior_zeta}
Distinct from the conventional conservative formulation,
the contraction criterion \eqref{eq:relaxed_zeta_condition} ensures that a strictly positive upper bound for $\zeta$ is always well-defined even if $p_{11} \to 1$.
Nevertheless, the transition probability $p_{11}$ cannot be arbitrarily close to unity due to instantaneous state constraint considerations.
In practice, the maximum allowable open-loop prediction time $l_{max} \triangleq \sup \left\{ l \in \mathbb{Z}_{> 0} \ \middle\vert\ \mathcal{X}_{safe} \ominus \mathcal{R}_l \neq \emptyset \right\}$ limits the valid upper bound of the transition probability via $p_{11} \le 1 - \frac{1}{l_{max}}$.
Consequently, the finite ultimate boundary $E_{\infty}^2 < \infty$ remains valid exclusively within the system's feasible operational regime, defined by a compact subset where the transition probability $p_{11}$ is bounded away from 1.
\end{remark}

\subsection{Probabilistically Truncated Soft Constraints and Recursive Feasibility}
\label{subsec:soft_constraints_feasibility}

Although the high-dimensional prediction error achieves global MSUB with its steady-state mean-square limit bounded by
$E_{\infty}^2 \triangleq c_2 M_w / (c_1 \alpha)$,
the convex QP solver cannot directly process high-dimensional expectations.
To synthesize executable control laws,
these statistical bounds must be mapped into algebraic constraints on the $n_x$ dimension state space.

Define the prediction error in the state space as $e_{x, k} = C e_k$.
Given that the spectral norm of the observation matrix $C = [I_{n_x}, \mathbf{0}]$ satisfies $\|C\|_2 = 1$,
the expected squared norm of the state space error is upper-bounded by the high-dimensional mean-square limit:
\begin{equation}
\limsup_{k \to \infty} \mathbb{E} [ \|e_{x, k}\|_2^2 ] \le \limsup_{k \to \infty} \mathbb{E} [ \|C\|_2^2 \|e_k\|_2^2 ] \le E_{\infty}^2
\end{equation}

Extending this asymptotic property to ensure safety at any transient finite time $k$ requires an analysis of the system's finite-time behavior.
As assumed in Section \ref{subsec:open_loop_expansion},
since the initial prediction error resides within the bounded compact set $\mathcal{E}_0$,
the expected squared error at the initial time remains bounded.
Coupled with the single-step conditional expected dissipation property established in Section \ref{subsec:msub_markov_jumps},
the sequence of expected squared errors $\mathbb{E}[\|e_k\|_2^2]$ will not experience finite-time escape.
Consequently, there definitively exists a global absolute constant bound
$\bar{E}_{x}^2 \triangleq \max\left( \sup_{k \ge 0} \mathbb{E}[\|e_{x, k}\|_2^2], E_{\infty}^2 \right) < \infty$,
which envelops both the maximum initial transient and the steady state limit,
such that for any evolution step $k \in \mathbb{Z}_{\ge 0}$, the following holds:
\begin{equation}
    \mathbb{E}[\|e_{x, k}\|_2^2] \le \bar{E}_{x}^2
\end{equation}

Since the squared $L_2$ norm $\|e_{x, k}\|_2^2$ is a non-negative random variable,
applying Markov's Inequality dictates that for any prescribed positive spatial boundary $R_{prob} > 0$,
the marginal probability of a constraint violation is bounded by:
\begin{equation}
    \operatorname{Pr}(\|e_{x, k}\|_2 \ge R_{prob}) =
    \operatorname{Pr}(\|e_{x, k}\|_2^2 \ge R_{prob}^2)
    \le \frac{\mathbb{E}[\|e_{x, k}\|_2^2]}{R_{prob}^2} 
    \le \frac{\bar{E}_{x}^2}{R_{prob}^2}.
\end{equation}

To guarantee instantaneous close-loop safety at a prescribed high confidence level $p \in (0, 1)$,
the upper bound of the violation probability is equated to $1-p$.
This formulation yields a distribution-free limit safety radius that is independent of
any specific non-Gaussian or mixed probability distribution morphology:
\begin{equation}\label{eq:Rprob}
    R_{prob} = \frac{\bar{E}_{x}}{\sqrt{1 - p}}
\end{equation}

To resolve the loss of recursive feasibility induced by continuous measurement dropouts as revealed in Section \ref{subsec:open_loop_expansion},
this steady state probabilistic boundary is utilized to truncate the deterministic dynamic contraction margins under open-loop evolution.
Let $l_k$ denote the accumulated number of continuous measurement dropout steps up to the current control step $k$.
Over the prediction horizon $i = 1, \dots, N_p$,
the truncated dynamic contraction margins for each dimension are reconstructed as:
\begin{equation}\label{eq:truncation_margins}
    \Delta \mathcal{X}_i^{*(j)} \triangleq \min(\Delta \mathcal{X}_{l_k+i}^{(j)}, R_{prob}), \quad j = 1, \dots, n_x
\end{equation}

\begin{remark}
\label{rem:probabilistic_truncation}
The truncation mechanism formulated in \eqref{eq:truncation_margins} enforces two relaxations to guarantee computational tractability.
Deploying $R_{prob}$ as an independent axis-wise safety margin inherits the $L_\infty$-norm conservative inner approximation detailed
in Remark \ref{rem:conservative_approximation}.
Probabilistically, the minimum operation enforces point-wise bounds rather than joint chance constraints over the prediction horizon.
While sacrificing exact multi-step joint probabilistic guarantees,
this decoupled approach converts complex stochastic envelopes into standard affine constraints,
confining the marginal violation risk at any individual dimension and prediction step to $1-p$.
\end{remark}

Furthermore, non-negative slack variable vectors
$\epsilon_{k+i|k} = [\epsilon_i^{(1)}, \dots, \epsilon_i^{(n_x)}]^T \ge \mathbf{0}$
are introduced to relax the hard state constraints via an exact penalty formulation:
\begin{equation}
x_{\min}^{(j)} + \Delta \mathcal{X}_i^{*(j)} - \epsilon_i^{(j)} \le c_j^T C \bar{z}_{k+i|k} \le x_{\max}^{(j)} - \Delta \mathcal{X}_i^{*(j)} + \epsilon_i^{(j)}, \qquad j = 1, \dots, n_x
\end{equation}

Building upon this, the reformulated probabilistically truncated soft-constrained optimization problem,
denoted as $\mathcal{P}_N^{soft}(\bar{z}_k)$, is formulated as:
\begin{subequations}\label{eq:softOPT}
\begin{align}
    & \min_{U, \epsilon, \epsilon_{init}} \quad
    J_{soft} = J_{nominal} + \left( \lambda_{init}^T \epsilon_{init} + \epsilon_{init}^T S_{init} \epsilon_{init} \right) + \sum_{i=1}^{N_p-1} \left( \lambda^T \epsilon_{k+i|k} + \epsilon_{k+i|k}^T S \epsilon_{k+i|k} \right) \label{eq:softOPT_cost} \\
    & \quad \text{s.t.} \nonumber \\
    & \qquad \bar{z}_{k+i+1|k} = A \bar{z}_{k+i|k} + B u_{k+i|k}, ~ i = 0, \dots, N_p-1 \label{eq:softOPT_dyn} \\
    & \qquad -\epsilon_{init} \le \bar{z}_{k|k} - \bar{z}_k \le \epsilon_{init} \label{eq:softOPT_init} \\
    & \qquad x_{\min}^{(j)} + \Delta \mathcal{X}_i^{*(j)} - \epsilon_i^{(j)} \le c_j^T C \bar{z}_{k+i|k} \le x_{\max}^{(j)} - \Delta \mathcal{X}_i^{*(j)} + \epsilon_i^{(j)}, \quad j = 1, \dots, n_x, \; ~ i = 1, \dots, N_p-1 \label{eq:softOPT_state} \\
    & \qquad \bar{z}_{k+N_p|k} \in \mathcal{Z}_f \label{eq:softOPT_term} \\
    & \qquad u_{\min} \le u_{k+i|k} \le u_{\max}, ~ i = 0, \dots, N_p-1 \label{eq:softOPT_ctrl} \\
    & \qquad \epsilon_{init} \ge \mathbf{0}, \quad \epsilon_{k+i|k} \ge \mathbf{0}, ~ i = 1, \dots, N_p-1 \label{eq:softOPT_slack}
\end{align}
\end{subequations}
where $J_{nominal} \triangleq \sum_{i=0}^{N_p-1} \left( \| \bar{z}_{k+i|k} - z_{ref} \|_Q^2 + \| u_{k+i|k} \|_R^2 \right) + \| \bar{z}_{k+N_p|k} - z_{ref} \|_{P_f}^2$;
the parameter $\epsilon_{init} \in \mathbb{R}^{n_z}$ is the initial state relaxation slack variable;
the weight vectors $\lambda \gg \mathbf{0}$ and $\lambda_{init} \gg \mathbf{0}$ are sufficiently large constants designed to guarantee the exactness of the penalty function \cite{luque2025model,Kerrigan2000SoftCA};
and the diagonal matrices $S, S_{init} \succeq 0$ ensure the convexity of the QP solver.

Since the optimizer utilizes axis-aligned boundaries following the aforementioned $L_\infty$ norm conversion logic,
the limit probabilistic truncation envelope set under the $L_\infty$ norm is defined as
$\mathcal{B}_{prob}^{\infty} \triangleq \{ x \in \mathbb{R}^2 \mid \|x\|_{\infty} \le R_{prob} \}$.
To ensure problem solvability, the tightened constraint under the worst-case scenario is defined as
\begin{equation}\label{eq:Xtight}
    \mathcal{X}_{tight} \triangleq \mathcal{X}_{safe} \ominus \mathcal{B}_{prob}^{\infty},
\end{equation}
where $\mathcal{X}_{tight}$ must constitute a non-empty compact set.

Furthermore, the stability of the proposed predictive controller relies on the offline construction of the terminal ingredients.
The local feedback gain $K_f$ is derived by solving the unconstrained discrete-time LQR problem utilizing the nominal system matrices $(A, B)$.
Based on the corresponding DARE solution $P_f \succ 0$,
the nominal terminal set $\mathcal{Z}_f$ is constructed as a Lyapunov level set ensuring single-step cost decay.
To theoretically guarantee closed-loop safety and recursive feasibility,
$\mathcal{Z}_f$ is required to satisfy the following \textbf{terminal compatibility conditions}:
\begin{enumerate}[1)]
    \item \textbf{State Compatibility}:
    The projection of the latent terminal invariant set $\mathcal{Z}_f$ onto the original state space
    must be contained within the tightened safe constraint set, i.e.,
    $C \mathcal{Z}_f \subseteq \mathcal{X}_{tight}$.
    \item \textbf{Control Compatibility}:
    Under the local feedback law $u = K_f z$,
    the terminal set $\mathcal{Z}_f$ must satisfy the following conditions for any $z \in \mathcal{Z}_f$:
    1) \emph{Positive Invariance}: $(A + B K_f) \mathcal{Z}_f \subseteq \mathcal{Z}_f$;
    2) \emph{Control Admissibility}: The generated feedback control inputs must honor the amplitude and increment limits,
    formulated as $u_{\min} \le K_f z \le u_{\max}$.
\end{enumerate}

\begin{theorem}
\label{thm:recursive_feasibility}
For the optimization problem \eqref{eq:softOPT}, suppose the aforementioned terminal compatibility conditions are satisfied.
If an initially feasible solution exists at the initial time,
then under the stochastic measurement availability sequence governed by the defined Markov chain,
the following properties hold:
\begin{enumerate}[1)]
    \item \textbf{Recursive Feasibility}:
    The optimization problem \eqref{eq:softOPT} remains recursively feasible for all evolution steps $k \in \mathbb{Z}_{\ge 0}$,
    regardless of the measurement availability mode.
    \item \textbf{Expected Dissipation}:
    The optimal value function $V_N^0(\bar{z}_k)$ of the nominal dynamics acts as a stochastic Lyapunov function,
    satisfying the following expected dissipation inequality with a bounded drift term $M_V < \infty$:
    \begin{equation}
        \mathbb{E}[V_N^0(\bar{z}_{k+1}) \mid \bar{z}_k] - V_N^0(\bar{z}_k) \le -\lambda_{\min}(Q) \|\bar{z}_k - z_{ref}\|_2^2 + M_V
    \end{equation}
    \item \textbf{Closed-loop Tracking Error MSUB and Probabilistic Safety}:
    The true closed-loop tracking error $\tilde{z}_{true, k} \triangleq z_{true, k} - z_{ref}$ achieves global MSUB.
    The instantaneous marginal probability of the true state deviating from the nominal prediction
    beyond the probabilistically truncated safety radius is bounded by the prescribed confidence level $p$:
    \begin{equation}
        \operatorname{Pr}(\| C z_{true, k} - C \bar{z}_{k} \|_2 \ge R_{prob}) \le 1 - p
    \end{equation}
\end{enumerate}
\end{theorem}

\begin{proof}
\noindent \textbf{1) Recursive Feasibility:}
Assume that at the current time $k$,
an optimal nominal control sequence $U_k^* = [u_{k|k}^*, \dots, u_{k+N_p-1|k}^*]$ exists for the optimization problem.
At time $k+1$, the standard shifted candidate control sequence is constructed as follows:
\begin{equation}
    \tilde{U} = [u_{k+1|k}^*, \dots, u_{k+N_p-1|k}^*, K_f \bar{z}_{k+N_p|k}]
\end{equation}

To establish recursive feasibility,
it must be demonstrated that the shifted candidate sequence $\tilde{U}$ satisfies all constraints of the reformulated optimization problem
$\mathcal{P}_N^{soft}(\bar{z}_{k+1})$, which primarily comprise the hard control amplitude limits and the probabilistically truncated soft state constraints.
For the prediction steps $i \in [1, N_p-1]$, the candidate inputs $u_{k+i|k}^*$ inherit feasibility from the optimal sequence obtained at step $k$.
At the terminal step $i = N_p$,
the defined Control Compatibility condition ensures that the nominal terminal set $\mathcal{Z}_f$ is control-admissible,
guaranteeing that the local feedback $K_f \bar{z}_{k+N_p|k}$ honors the control amplitude boundaries.
Consequently, the candidate sequence $\tilde{U}$ satisfies the hard control constraints.

The nominal state update must be analyzed across two exclusive measurement conditions.
Denoting the open-loop state $\hat{z}_{k+1} \triangleq A \bar{z}_k + B u_{k|k}^*$,
the realized update is defined as $\bar{z}_{k+1} = \hat{z}_{k+1} + \delta_k$,
where $\delta_k$ represents the jump term.
According to the Markov jump error dynamics governed by \eqref{eq:mjls_error_dynamics}--\eqref{eq:equivalent_disturbance},
the jump term $\delta_k$ expands as:
\begin{subequations}
\begin{align}
    \delta_k &= \begin{cases}
        \mathbf{0}, & \text{if } \theta_{k+1} = 1 \\
        z_{true, k+1} - v_{k+1} - (A \bar{z}_k + B u_k^*), & \text{if } \theta_{k+1} = 0
    \end{cases} \label{eq:delta_first} \\
    &= \begin{cases}
        \mathbf{0}, & \text{if } \theta_{k+1} = 1 \\
        d_k - v_{k+1} + A e_k, & \text{if } \theta_{k+1} = 0
    \end{cases} \label{eq:delta_second}
\end{align}
\end{subequations}
Crucially, the recursive feasibility must be guaranteed under two distinct modes:

\textit{Under open-loop propagation} ($\theta_{k+1} = 1$): The nominal state evolves as dynamically predicted without jumps.
The shifted candidate sequence $\tilde{U}$ drives the terminal state into the invariant set $\mathcal{Z}_f$.
Meanwhile, over the prediction horizon $i \in [1, N_p-1]$,
$\Delta \mathcal{X}_i^{*(j)}$ may expand driven by accumulated measurement unavailability.
In this mode, $\epsilon_{k+i|k}$ are only required to compensate for
the continuous expansion of the dynamic contraction margins $\Delta \mathcal{X}^*_i$.

\textit{Under measurement-triggered resets} ($\theta_{k+1} = 0$):
The state reset introduces the jump term $\delta_{k}$ to the initial nominal state $\bar{z}_{k+1}$.
Instead of shifting the entire candidate trajectory,
the exact penalty framework accommodates this jump via $\epsilon_{init}$.
By allocating $\epsilon_{init}$ to absorb the reset jump,
the optimizer anchors the initial prediction state at $\hat{z}_{k+1}$.
Consequently, the original shifted candidate sequence $\tilde{U}$ propels the system into
the invariant terminal set $\mathcal{Z}_f$.

Consequently, regardless of whether the system experiences open-loop propagation or stochastic resets,
there always exists a bounded non-negative slack sequence that renders $\tilde{U}$ a globally feasible solution,
thereby preserving recursive feasibility.

\noindent \textbf{2) Expected Dissipation:}
The predicted nominal state sequence is confined within a compact reachable set,
formulated as $\bar{z}_{k+i|k} \in \bar{\mathcal{Z}}_{reach} \subset \mathbb{R}^{n_z}$ for all $i \in \{1, \dots, N_p\}$.
According to Lemma \ref{lem:single_step_bounds},
the true state tracking error is prevented from unbounded finite-time escape.
For any time step $k$, there definitively exists an upper bound $E_{max} < \infty$
such that the error satisfies $\|z_{true, k} - \bar{z}_{k|k}\|_2 \le E_{max}$.
Confined within the compact reachable set $\bar{\mathcal{Z}}_{reach}$,
the maximum nominal state constraint violation induced by the candidate sequence defined as
the margin by which the nominal predicted states exceed the tightened boundaries $\mathcal{X}_{tight}$,
possesses a finite upper bound.
Let the nominal state deviation at the current time be defined as $\tilde{z}_k \triangleq \bar{z}_k - z_{ref}$,
and let the optimal value function be denoted by $V_N^0(\bar{z}_k)$.
Therefore, there exists a finite constant $\Delta V_{soft}^{max} < \infty$ triggered by the worst-case activation of the soft-constraint slack variables.
Let $J_N(\hat{z}_{k+1}, \tilde{U})$ denote the objective function cost corresponding to the shifted candidate sequence $\tilde{U}$
when applied to $\hat{z}_{k+1}$.
Since the terminal control law $K_f$ guarantees a quadratic reduction in the nominal terminal cost,
the resulting cost descent property of the candidate sequence is formulated as:
\begin{equation}
    J_N(\hat{z}_{k+1}, \tilde{U}) \le V_N^0(\bar{z}_k) - \lambda_{\min}(Q) \|\tilde{z}_k\|_2^2 + \Delta V_{soft}^{max}
\end{equation}

Given that the optimal value function is defined as the minimum over all feasible sequences,
it satisfies $V_N^0(\hat{z}_{k+1}) \le J_N(\hat{z}_{k+1}, \tilde{U})$.
This establishes the descent bound for the optimal cost:
\begin{equation}
    V_N^0(\hat{z}_{k+1}) \le V_N^0(\bar{z}_k) - \lambda_{\min}(Q) \|\tilde{z}_k\|_2^2 + \Delta V_{soft}^{max}
\end{equation}

A Convex QP featuring positive definite quadratic forms and an $L_1$ exact penalty exhibits continuity over the parametric reachable set.
Hence, the value function $V_N^0(\cdot)$ is locally Lipschitz continuous, with a Lipschitz constant denoted as $L_V > 0$.
The triangle inequality yields:
\begin{equation}
    V_N^0(\bar{z}_{k+1}) \le V_N^0(\hat{z}_{k+1}) + L_V \|\delta_k\|_2
\end{equation}

Since Theorem \ref{thm:msub} proved that $e_k$ achieves MSUB,
and $d_k \in \mathcal{D}, v_{k+1} \in \mathcal{E}_0$
are bounded disturbances within compact sets,
the expectation of the nominal state jump is bounded,
establishing $\mathbb{E}[\|\delta_k\|_2] \le M_\delta$ for a finite constant $M_\delta < \infty$.
By taking the conditional expectation on both sides and combining the constant terms,
the steady-state expected jump constant is defined as $M_V \triangleq \Delta V_{soft}^{max} + L_V M_\delta < \infty$.
This formulation yields the expected descent inequality under random resets:
\begin{equation}\label{eq:expected_descent_inequality}
    \mathbb{E}[V_N^0(\bar{z}_{k+1}) \mid \bar{z}_k] - V_N^0(\bar{z}_k) \le -\lambda_{\min}(Q) \|\tilde{z}_k\|_2^2 + M_V
\end{equation}

\begin{figure*}\rmfamily
  \centering
  \begin{subfigure}[t]{0.48\linewidth}
    \centering
    \includegraphics[width=\linewidth]{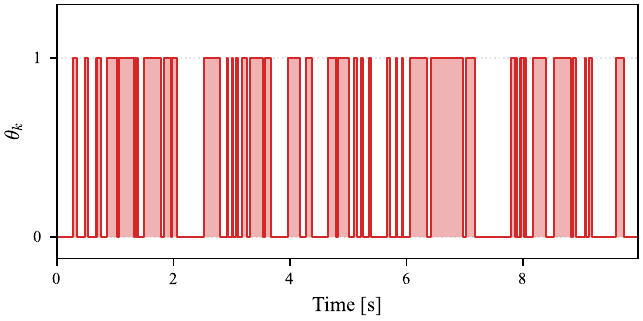}
    \caption{}
    \label{fig:subd}
  \end{subfigure}
  \hfill
  \begin{subfigure}[t]{0.48\linewidth}
    \centering
    \includegraphics[width=\linewidth]{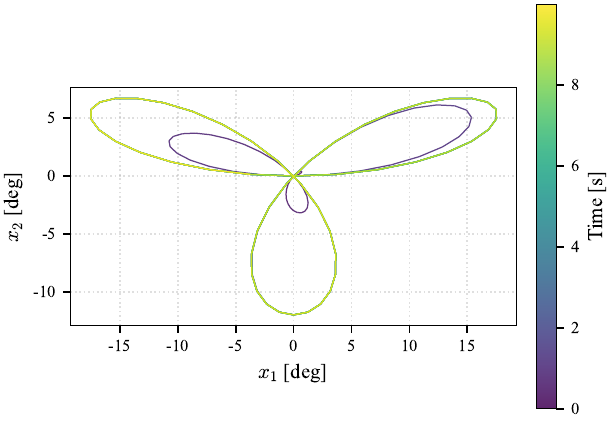}
    \caption{}
    \label{fig:subc}
  \end{subfigure}
  \caption{Numerical validation setup: (a) markov jump sequence of measurement availability; (b) reference trajectory in the image plane.}
  \label{fig:composite1}
\end{figure*}

\noindent \textbf{3) Closed-loop Tracking Error MSUB and Probabilistic Safety:}
The optimal value function $V_N^0(\bar{z}_k)$ is bounded by positive definite quadratic functions,
satisfying $c_{v_{1}} \|\tilde{z}_k\|_2^2 \le V_N^0(\bar{z}_k) \le c_{v_{2}} \|\tilde{z}_k\|_2^2$ for some constants $c_{v_{1}},c_{v_{2}}>0$.
Taking the unconditional expectation on both sides of ~\eqref{eq:expected_descent_inequality}
as $k \to \infty$, the mean-square value of $\tilde{z}_k$ is confined within an ultimate upper bound:
\begin{equation}
    \limsup_{k \to \infty} \mathbb{E} [ \Vert{}\tilde{z}_k\Vert{}_2^2 ] \le \frac{c_{v_{2}} M_V}{c_{v_{1}} \lambda_{\min}(Q)} < \infty
\end{equation}

According to the principle of linear superposition,
the true closed-loop tracking error can be decomposed as $z_{true, k} - z_{ref} = \tilde{z}_k + e_{z, k}$.
Since $\tilde{z}_k$ achieves mean-square boundedness,
and the global MSUB property of $e_{z, k}$ has been established in Theorem \ref{thm:msub},
$\tilde{z}_{true, k}$ achieves global MSUB.

Furthermore, regarding the probabilistic safety,
the deviation of the true state from the updated nominal prediction in the original state space
is the low-dimensional projection of the prediction error.
As derived in Section \ref{subsec:soft_constraints_feasibility},
the uniformly bounded expected squared error $\mathbb{E}[\|C z_{true, k} - C \bar{z}_k\|_2^2] \le \bar{E}_{x}^2$ guarantees that
the marginal violation probability against the truncated safety radius satisfies $\operatorname{Pr}(\| C z_{true, k} - C \bar{z}_k \|_2 \ge R_{prob}) \le 1 - p$.
This completes the proof.
\end{proof}

\section{Numerical Validation}
\label{sec:simulation}

To verify the obtained theoretical results, a visual servoing tracking task for a 2-DOF flexible-joint gimbal driven by series elastic actuators (SEAs) is adopted in this section.
While the foundational SEA dynamic structure is inspired by the model in \cite{zhao2024novel},
the physical parameters are specifically adapted to reflect a realistic gimbal.
Accordingly, the nonlinear dynamics of the system are formulated as follows:
\begin{equation}
\begin{aligned}
    J_p \ddot{q}_1 + c_p \dot{q}_1 + k_p q_1 + C_1(q_1, q_2, \dot{q}_1, \dot{q}_2) &= \tau_1, \\
    J_t \ddot{q}_2 + c_t \dot{q}_2 + k_t q_2 + C_2(q_1, q_2, \dot{q}_1) + G_{2}(q_2) &= \tau_2,
\end{aligned}
\end{equation}

where $q_1$ and $q_2$ represent the angular positions of the pan and tilt axes, respectively,
while $\tau_1$ and $\tau_2$ denote the control torques applied to respective axes.
The parameters $J_p = 0.05 \text{ kg} \cdot \text{m}^2$ and $J_t = 0.02 \text{ kg} \cdot \text{m}^2$ denote the moments of inertia for the pan and tilt axes, respectively.
The terms $k_p = 3.0 \text{ Nm/rad}$ and $k_t = 1.5 \text{ Nm/rad}$ represent the elastic coefficients,
while $c_p = 0.5 \text{ Nm} \cdot \text{s/rad}$ and $c_t = 0.3 \text{ Nm} \cdot \text{s/rad}$ account for the viscous damping coefficients.
Furthermore, the term $C_1(\cdot)$ accounts for the Coriolis coupling effects,
$C_2(\cdot)$ encapsulates the centrifugal coupling term,
and $G_{2}(q_2)$ models the gravitational unbalance torque.

\begin{figure*}
  \centering
  \begin{subfigure}[t]{0.48\linewidth}
    \centering
    \includegraphics[width=\linewidth]{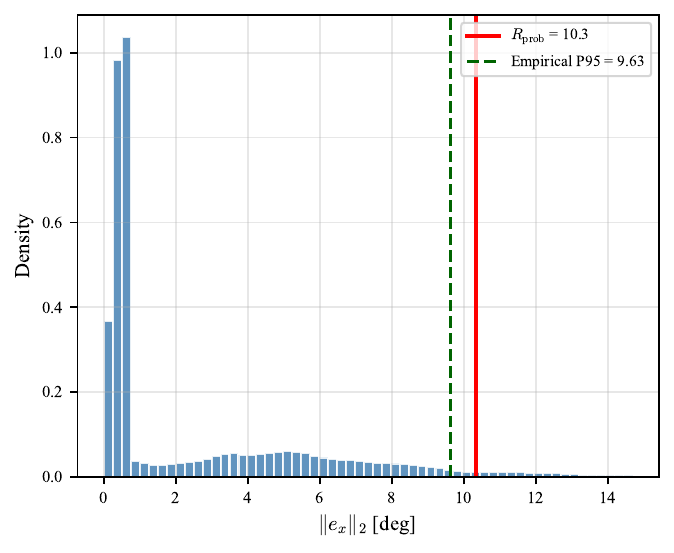}
    \caption{}
    \label{fig:error_distribution}
  \end{subfigure}
  \hfill
  \begin{subfigure}[t]{0.48\linewidth}
    \centering
    \includegraphics[width=\linewidth]{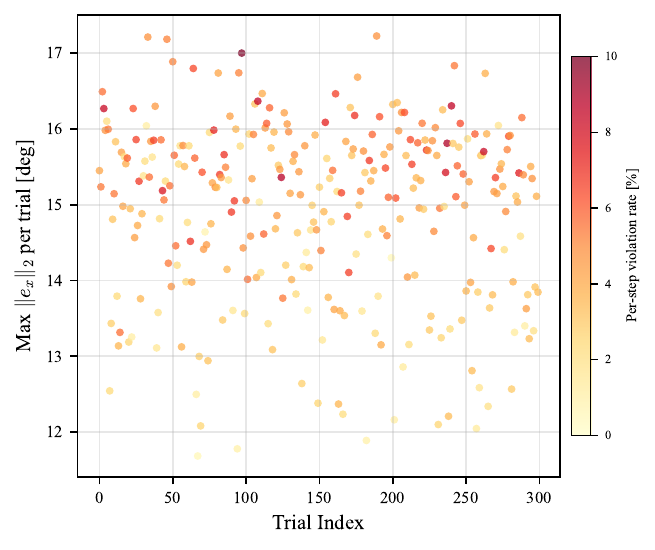}
    \caption{}
    \label{fig:max_error}
  \end{subfigure}
  \caption{Statistical validation of the probabilistic safety boundary via Monte Carlo simulations:
  (a) density distribution of the prediction error;
  (b) maximum prediction error and per-step violation rate across $300$ trials.}
  \label{fig:MTKL}
\end{figure*}

Concurrently, to replicate the stochastic measurement dropouts frequently encountered in practical visual tracking systems,
the measurement availability mode $\theta_k$ is modeled as a homogeneous discrete-time Markov chain,
a sample sequence of which is displayed in Figure~\ref{fig:subd}.
The transition probabilities are configured as $p_{01} = 0.15$ and $p_{11} = 0.8$.

Furthermore, Figure~\ref{fig:subc} illustrates the predefined reference trajectory within the $x_1$-$x_2$ image feature plane, which is parameterized as a three-petal rose curve:
$
x_{1,\text{ref}}(t) = 20^\circ\sin(3\pi t) \cos(\pi t),
x_{2,\text{ref}}(t) = 12^\circ \sin(3\pi t) \sin(\pi t).
$
Constrained by the FOV of the onboard camera,
the safe operational state space of the system is defined as a compact boundary:
$\mathcal{X}_{\text{safe}} = \{ x \in \mathbb{R}^4 \mid {} |x_1| \le 38^\circ,\ |x_2| \le 30^\circ, |x_3| \le 261.8\,^\circ/\mathrm{s},\ |x_4| \le 143.5\,^\circ/\mathrm{s} \}$.
The dual-axis control torques are restricted by $|u| \le 3.0 \text{ Nm}$.
The prediction horizon of the MPC controller is chosen as $N_p = 10$.
The quadratic cost weighting matrices for the system states and control inputs are configured as $Q = \operatorname{diag}(100, 200, 10, 10)$
and $R = 0.5 I_{2\times2}$, respectively,
where a minor regularization $q_\psi = 10^{-3}$ is applied to the latent states.
To implement the exact soft-penalty mechanism, the quadratic penalty matrix for the slack variables is set to $S = \operatorname{diag}(1, 1, 0.1, 0.1)$,
and a sufficiently large exact penalty weight vector is configured as $\lambda = [500, 500, 1000, 1000]^T$.
Furthermore, the initial exact penalty weight vector is specified as $\lambda_{init} = [\lambda^T, 500 \cdot \mathbf{1}_{n_\psi}^T]^T \in \mathbb{R}^{n_z}$,
and the corresponding initial quadratic penalty matrix is formulated as $S_{init} = \operatorname{diag}(S, I_{n_\psi}) \in \mathbb{R}^{n_z \times n_z}$.
The comprehensive structural and training configurations of the deep Koopman operator network are detailed in Table~\ref{tab:koopman_network_parameters}.

To quantify the probabilistic safety boundary, assuming that the Deep Koopman operator perfectly captures the nonlinear dynamics,
the equivalent expected disturbance bound is determined as $M_w = 0.000426$ under the dynamic drift $r_w = 0.002$ and the sensor noise $r_{sensor} = 0.012$.
Consequently, the steady-state mean-square error limit defined in Theorem~\ref{thm:msub} is computed as $E_{\infty}^2 = 0.001626 \text{ rad}^2$,
derived from the Lyapunov condition bounds $c_1 = 1.01$ and $c_2 = 3.84$, and the global dissipation rate $\alpha = 1.00$.
Moreover, under a confidence level of $p = 0.95$,
the distribution-free probabilistic safety radius~\eqref{eq:Rprob} is determined as $R_{\text{prob}} = 10.33^\circ$,
which satisfies the prerequisite that the tightened safe space $\mathcal{X}_{\text{tight}}$ defined in~\eqref{eq:Xtight} remains a non-empty compact set,
thereby guaranteeing the applicability of the proposed control scheme.
To empirically validate this theoretical boundary,
a Monte Carlo simulation consisting of $300$ independent trials is conducted.
As illustrated in Figure~\ref{fig:error_distribution},
the density distribution of the tracking error $\|e_x\|_2$ reveals that the empirical 95th percentile (P95) is $9.63^\circ$,
which is upper-bounded by the analytical theoretical radius $R_{prob} = 10.33^\circ$.
Furthermore, Figure~\ref{fig:max_error} visualizes the maximum $\|e_x\|_2$ per trial across the $300$ trials,
with the color map indicating the per-step constraint violation rate for each trial.
Statistical results demonstrate that the overall empirical violation rate across all simulation steps is $4.10\%$,
successfully honoring the predefined $5\%$ risk tolerance limit.

\begin{figure*}
  \centering
  \begin{subfigure}[t]{0.48\linewidth}
    \centering
    \includegraphics[width=\linewidth]{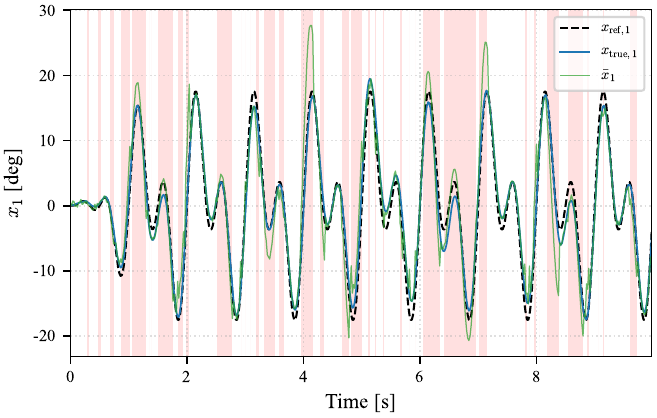}
    \caption{}
    \label{fig:sube}
  \end{subfigure}
  \hfill
  \begin{subfigure}[t]{0.48\linewidth}
    \centering
    \includegraphics[width=\linewidth]{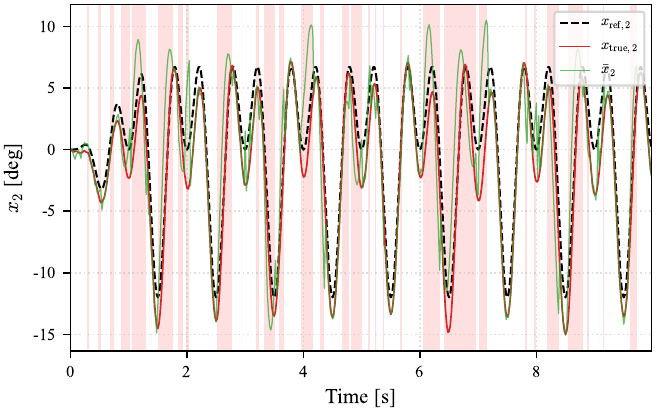}
    \caption{}
    \label{fig:subf}
  \end{subfigure}
  \\
  \begin{subfigure}[t]{0.48\linewidth}
    \centering
    \includegraphics[width=\linewidth]{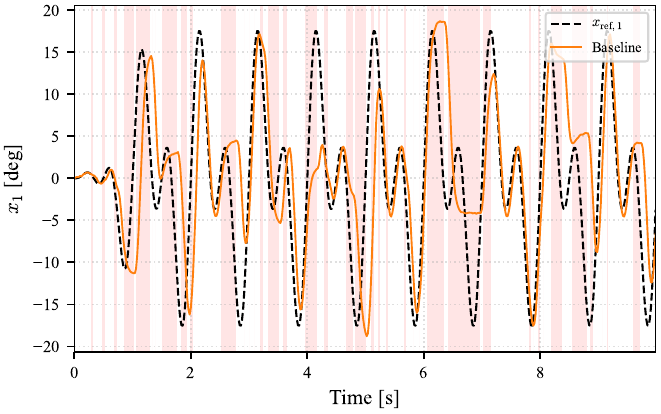}
    \caption{}
    \label{fig:subb2}
  \end{subfigure}
  \begin{subfigure}[t]{0.48\linewidth}
    \centering
    \includegraphics[width=\linewidth]{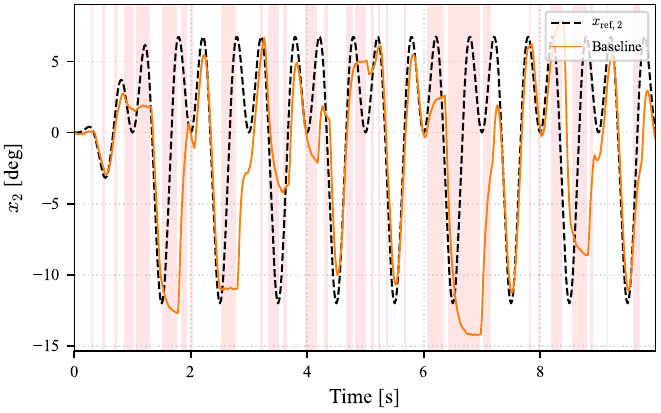}
    \caption{}
    \label{fig:subb3}
  \end{subfigure}
  \caption{Tracking performance of the system states under stochastic intermittent measurements:
  (a) pan angle $x_1$;
  (b) tilt angle $x_2$;
  (c) baseline pan angle $x_1$;
  (d) baseline tilt angle $x_2$.}
  \label{fig:composite2}
\end{figure*}

Figures~\ref{fig:sube} and \ref{fig:subf} depict the dynamic responses of the system under stochastic measurement dropouts (indicated by the red shaded regions).
Overall, the proposed approach achieves root-mean-square errors (RMSEs) of $1.5^\circ$ and $1.8^\circ$ for the pan and tilt axes, respectively.
During intervals of measurement availability, the system continuously acquires sensor updates,
steering the true state trajectory $x_{\text{true}}$ to closely match the reference trajectory $x_{\text{ref}}$,
which yields mean absolute errors (MAEs) of $1.1^\circ$ and $1.5^\circ$ for the pan and tilt axes.
Conversely, when $\theta_k = 1$ (i.e., measurements are unavailable),
the controller switches to the open-loop nominal prediction mode,
where the MAEs are maintained at $1.2^\circ$ and $1.6^\circ$ for the pan and tilt axes, respectively.

To further quantify the performance superiority of the proposed algorithm,
a comparative study is conducted between the presented deep Koopman soft-constrained MPC and a switched PD controller with ZOH baseline approach~\cite{liang2023tracking}.
Quantitative statistics demonstrate that during intervals of measurement availability,
the baseline controller yields MAEs of $5.1^\circ$ and $2.6^\circ$ for the pan and tilt axes, respectively.
Compared to these values, the proposed method achieves error reductions of $78.4\%$ and $42.3\%$.
Furthermore, during intervals of measurement unavailability, the baseline MAEs surge to $6.9^\circ$ and $5.1^\circ$ for the respective axes.
By leveraging the open-loop nominal state evolution, the proposed method achieves reductions of $82.6\%$ and $68.6\%$ in these dropout intervals.
This improvement is corroborated by the absolute tracking error comparison for the pan axis shown in Figure~\ref{fig:sube2}.
Additionally, the cumulative distribution function (CDF) curve of the tracking error (Figure~\ref{fig:subf2}) climbs much more rapidly within the low-error region than that of the baseline,
demonstrating that the tracking errors are tightly bounded within a narrow margin at a high confidence level.
Overall, compared to the baseline RMSEs of $8.4^\circ$ and $6.0^\circ$ for the pan and tilt axes, the proposed method realizes error reductions of $82.1\%$ and $70.0\%$, respectively.

The statistical results of the convex optimization solver across the Monte Carlo trials indicate a median single-step computation time of $2.0\text{ ms}$.
\footnote{All numerical simulations and performance validations are executed within a Python simulation environment running on an AMD Ryzen AI 9 H 365 CPU with 32 GB RAM.}
Considering the $50\text{ Hz}$ sampling frequency of the system, this execution speed satisfies the real-time control requirements.
This confirms the practical engineering feasibility of the proposed method, highlighting its potential for deployment on edge computing devices.

\begin{figure*}
  \centering
  \begin{subfigure}[t]{0.48\linewidth}
    \centering
    \includegraphics[width=\linewidth]{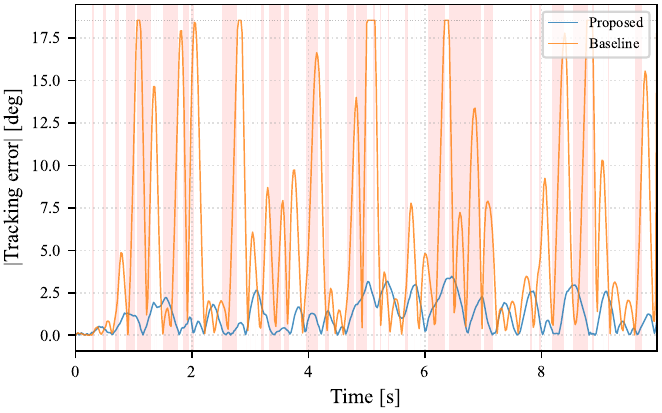}
    \caption{}
    \label{fig:sube2}
  \end{subfigure}
  \hfill
  \begin{subfigure}[t]{0.48\linewidth}
    \centering
    \includegraphics[width=\linewidth]{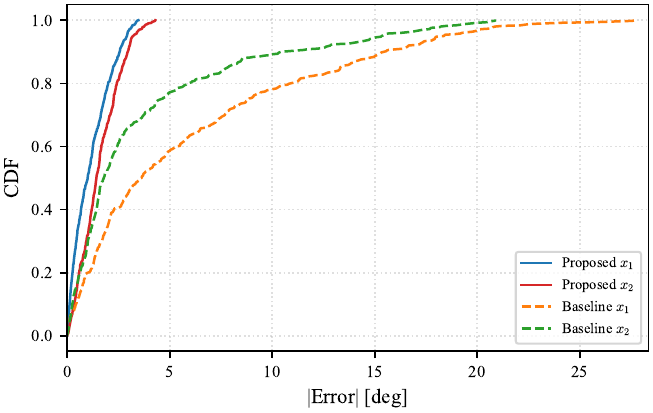}
    \caption{}
    \label{fig:subf2}
  \end{subfigure}
  \caption{Performance comparison with the baseline controller under identical stochastic intermittent measurements:
  (a) absolute tracking error comparison;
  (b) CDF of the tracking error.}
  \label{fig:composite3}
\end{figure*}

\section{Conclusion}
\label{sec:conclusion}
In conclusion, this paper concentrates on the problem of ensuring stability and recursive feasibility for constrained nonlinear dynamical systems
subjected to stochastic intermittent measurements.
A Lipschitz-constrained deep Koopman operator is designed to globally lift the nonlinear dynamics into a computationally tractable linear latent space.
Subsequently, the tracking process under unpredictable measurement dropouts is reformulated as a switched linear system,
and the high-dimensional prediction error is proven to achieve MSUB.
By mapping the steady-state statistical mean-square limit back into the original state space,
a distribution-free probabilistic safety radius is established.
Finally, by integrating this probabilistic boundary with an exact penalty function, a soft-constrained MPC framework is synthesized,
which guarantees recursive feasibility and MSUB while maintaining high computational efficiency,
thereby demonstrating potential for deployment on practical edge computing devices.

\bmsection*{Acknowledgments}
This work was supported in part by the National Natural Science Foundation of China under Grant 62225305, Grant 62527807 and Grant 62403169, in part by the Postdoctoral Fellowship Program of CPSF under Grant GZB20240960, in part by the State Key Laboratory of Robotics and Systems (HIT) under Grant SKLRS202501A04, in part by the State Key Laboratory of Robotics under Grant 2024-O10.

\bibliography{references}

\begin{appendix}
\bmsection{Network Parameters and Configuration\label{app:network_params}}

\begin{table}[!h]
\centering
\caption{Parameters of the Deep Koopman Operator Network and Training Configuration}
\label{tab:koopman_network_parameters}
\begin{tabular*}{\textwidth}{@{\extracolsep\fill}lc@{\extracolsep\fill}}
\toprule
\textbf{Parameter} & \textbf{Value / Bound} \\
\midrule
State dimension $n_x$ & 4 \\
Control input dimension $n_u$ & 2 \\
Linear latent space dimension $n_z$ & 16 \\
Initial learning rate & $1 \times 10^{-3}$ \\
Discount factor $\gamma$ & 0.90 \\
Eigenvalue penalty weight $\alpha_{eig}$ & 5.0 \\
Orthogonal regularization weight $\alpha_{ortho}$ & 4.0 \\
Maximum singular value of hidden layers $\sigma_i$ & 1.0 \\
Global Lipschitz constant of feature mapping $L_\psi$ & $\le 1$ \\
Initial maximum singular value of matrix $A$: $\sigma_{\max}$ & $\le 0.90$ \\
Eigenvalue envelope boundary $\beta$ & 0.92 \\
Steady-state spectral radius of matrix $A$: $\rho(A)$ & 0.8976 \\
\bottomrule
\end{tabular*}
\end{table}

\end{appendix}

\end{document}